\documentclass{article}
\usepackage{ijcai19}

\usepackage{times}
\usepackage{soul}
\usepackage{url}
\usepackage[hidelinks]{hyperref}
\usepackage[utf8]{inputenc}
\usepackage[small]{caption}
\usepackage{graphicx}
\usepackage{amsmath}
\usepackage{booktabs}
\usepackage{algorithm}
\usepackage[noend]{algorithmic}

\usepackage{microtype}      
\usepackage{amsthm}
\usepackage{amsfonts}

\usepackage{subcaption}
\usepackage{xcolor}
\usepackage{booktabs}
\usepackage{varwidth}
\usepackage{mathtools}
\usepackage{tikz}
\usetikzlibrary{shapes,arrows}
\usepackage{changepage}

\newtheorem{theorem}{Theorem}[section]
\newtheorem{corollary}{Corollary}[section]
\newtheorem{lemma}{Lemma}[section]
\newtheorem{definition}{Definition}

\newcommand{\rom}[1]{\uppercase\expandafter{\romannumeral #1\relax}}

\newcommand{\citep}{\cite}

\title{Confirmatory Bayesian Online Change Point Detection \\in the Covariance Structure of Gaussian Processes}

\author{
Jiyeon Han\thanks{Equally contributed}\and
Kyowoon Lee\footnotemark[1] \and
Anh Tong\And
Jaesik Choi\footnote{Contact Author}\\
\affiliations
Ulsan National Institute of Science and Technology\\
\emails
\{jiyeon, leekwoon, anhth, jaesik\}@unist.ac.kr
}

\begin{document}

\maketitle

\begin{abstract}
  In the analysis of sequential data, the detection of abrupt changes is important in predicting future events. 
In this paper, we propose statistical hypothesis tests for detecting covariance structure changes in locally smooth time series modeled by Gaussian Processes (GPs). We provide theoretically justified thresholds for the tests, and use them to improve Bayesian Online Change Point Detection (BOCPD) by confirming statistically significant changes and non-changes. Our Confirmatory BOCPD (CBOCPD) algorithm finds multiple structural breaks in GPs even when hyperparameters are not tuned precisely. We also provide conditions under which CBOCPD provides the lower prediction error compared to BOCPD. Experimental results on synthetic and real-world datasets show that our proposed algorithm outperforms existing methods for the prediction of nonstationarity in terms of both regression error and log likelihood.
\end{abstract}

\section{Introduction}

In sequential data, the \emph{change point detection} (CPD) problem, i.e., analysis of the data with the aim of detecting abrupt changes, is an important component in improving the prediction of future events. When the underlying distribution does not change (stationary), predicting future events becomes tractable. However, the stationarity assumption does not hold in practice.

A \emph{change point} (CP) is a specific sequential position at which the underlying distribution changes. CPs play critical roles in numerous real-world applications, including climate modeling \cite{manogaran2018spatial}, speech recognition \cite{panda2016automatic}, image analysis \cite{tewkesbury2015critical}, and human activity recognition \cite{human_activity_1}. In econometrics, \emph{structural breaks}, which essentially apply CPD to regression models, have been studied for decades in an attempt to identify structural stability in the forecasting of time series \cite{chow1960tests,ewing2016volatility}. Trend filtering determines CPs by assuming piecewise linearity in the sequential data \cite{kim2009ell_1}. CPD also plays an important role in the domain adaptation problem, where it is known as a covariate shift \cite{sugiyama2008direct}, as the distribution of the test data changes from that of the training data.   

Most existing CPD methods are based on either statistical (Bayesian) inference or hypothesis tests. Statistical inference methods compute the probability of the occurrence of a change. Bayesian CPD algorithms \cite{barry1993(bayesian-offline),xuan2007(bayesian-offline)} identify CPs using Bayesian framework. BOCPD algorithm \cite{BayesianOnlineChangepoint,garnett2010sequential(michael),GP-BOCPD,kim2015reading} detects CPs sequentially considering the correlated interval, the so-called \emph{run length}, between CPs. Such probabilistic methods, however, do not provide a statistical error bound leading to a lack of reliability and are highly sensitive to selected hyperparameters. 

Hypothesis-test-based approaches determine the existence of changes based on a statistical test, where the error probability is naturally determined during the computation. Representative hypothesis test techniques include kernel methods such as two-sample tests based on the maximum mean discrepancy \cite{KernelTwoSampleTest,M-Statistic}, kernel Fisher discriminant ratio \cite{KernelFisherDiscriminant},
and likelihood ratio tests such as the cumulative sum (CUSUM) test \cite{chernoff1964estimating(CUSUM_mean_1),gombay1996estimators(CUSUM_variance_2)}. 

While the conventional GPs only deal with stationary (globally smooth) functions, GPs with CPs can mimic locally smooth functions \cite{GP-BOCPD}, allowing them to represent many real-world time series data. There have been a number of studies of CPD in GPs using hypothesis tests. For instance, CPD in GPs using p-value test on the likelihood given the null hypothesis of stationary time series have been investigated \cite{isupova2017machine}. If the null distribution is not legitimate, this approach cannot decisively say there is a change even if the null hypothesis is rejected. Another work has studied to detect the mean changes in GPs through likelihood ratio test \cite{keshavarz2015optimal}.   

\begin{figure*}[t!]
\begin{center}
\includegraphics[width=\textwidth]{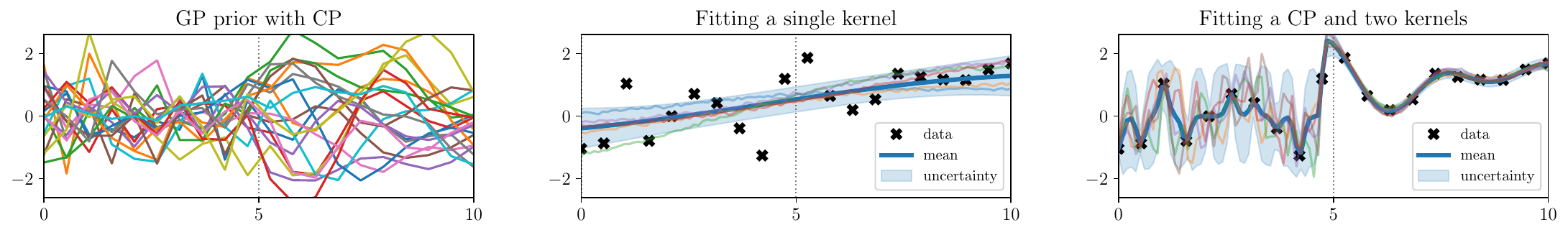}
    \setlength{\belowcaptionskip}{-10pt}
\caption{\textbf{Left}: GP prior with a covariance structure change. \textbf{Middle}: GP posterior after fitting data with single kernel. \textbf{Right}: GP posterior after fitting a CP and two kernels.}
\label{fig:change-point}

\end{center}
\end{figure*}
\raggedbottom

In this paper, we propose new statistical likelihood ratio tests that detect changes in the covariance structure of GPs and build a theoretically justified online detection algorithm, Confirmatory BOCPD, which detects CPs with a reasonable time delay. We also present sufficient conditions under which CBOCPD provides the lower prediction error compared to BOCPD. Moreover, our algorithm adjusts the parameter of BOCPD to avoid false alarms and missed detections when the results of hypothesis tests are sound. When the results are not sound, our algorithm takes advantages of Bayesian inference from BOCPD.



\section{Background}


\subsection{GP Models and Assumptions}

A GP is a random process produced by a collection of random variables such that any finite set of those variables follows a multivariate Gaussian distribution. A GP is completely specified by the mean function $\mu(\cdot)=\mathbb{E}[f(\cdot)]$ and the kernel function $K(\cdot,\cdot)=\text{Cov}(f(\cdot),f(\cdot))$. The kernel function describes distinctive characteristics of the sequential data, such as variance, length scales, and periodicity. For example, in the case of the Radial Basis Function kernel $K(t,t')=\sigma^2\exp(-\frac{(t-t')^2}{2l^2})$, the length scale hyperparameter $l$ controls the smoothness of the sequential data. 


In this paper, we assume that observations are corrupted by the white Gaussian noise with the variance $\sigma_{no}^2$. In modeling sequential data using GPs, the index $t$ is considered as the input, and the corresponding output is modeled as $x_t \sim N(f(t), \sigma_{no}^2)$. Given the GP hyperparameters $\theta_m$ and $\sigma_{no}$, the log marginal likelihood of the GP over the observed samples $\mathbf{x}$ is specified by the mean vector $\mathbf{\mu}$ and the covariance matrix $\Sigma$ for the multivariate Gaussian distribution.\footnote{$\log p(\mathbf{x}|\theta_m, \sigma_{no}) =  -\frac{1}{2}(\mathbf{x}{-}\mathbf{\mu})^T(\Sigma+\sigma_{no}^2 I)^{-1}(\mathbf{x}{-}\mathbf{\mu}) -\frac{1}{2}\log|\Sigma+\sigma_{no}^2 I|-\frac{n}{2}\log2\pi$, where $n$ is the number of observed samples.}

\subsection{Optimal CPD of the Mean}

Here, we briefly review some results on the detection of a single CP in the mean function of a GP \cite{keshavarz2015optimal}. The goal is to set a formal statistical test that can identify a change in the mean function and to quantify the confidence of a change.



We denote the $n$ observed sequential data as $X = \{\it{X}_t\}_{t=\text{1}}^n $. 
Let $t\in \mathcal{C}_n \subseteq \{1,...,n\}$ represents the point at which a sudden change occurs. 
Given sequential points, two hypotheses are used for the likelihood ratio test. 
One is the null hypothesis, $\mathbb{H}_0$, which insists there is no CP. The other is the alternative hypothesis, $\mathbb{H}_1$, which declares there is at least one CP. The hypothesis testing problem is constructed as follows. Given the two hypotheses, the likelihood ratio is defined as $2\mathfrak{L}=2(\text{sup}_{\theta\in\Theta_1}\ell(\theta_1)-\text{sup}_{\theta\in\Theta_0}\ell(\theta_0))$ where $\ell$ is the log likelihood function and $\Theta_0$ and $\Theta_1$ are the parameter spaces of $\mathbb{H}_0$ and $\mathbb{H}_1$ respectively.
The generalized likelihood ratio test (GLRT) is formulated as
$
\mathfrak{T}_{GLRT} = \mathbb{I}\left(2\mathfrak{L}\ge \mathfrak{R}_{n,\delta} \right)
$
with a proper threshold $\mathfrak{R}_{n,\delta}$ and the indicator function $\mathbb{I}(\cdot)$. Here, $\delta$ is the upper bound of the corresponding conditional detection error probability $\varphi_n(\mathfrak{T})$, which is defined as 
\begin{equation}
\varphi_n(\mathfrak{T})=\mathbb{P}(\mathfrak{T}=1|\mathbb{H}_0)+\max_{t\in \mathcal{C}_n}\mathbb{P}(\mathfrak{T}=0|\mathbb{H}_{1,t}).
\end{equation}
We reject $\mathbb{H}_0$ when $\mathfrak{T}_{GLRT}{=}1$, otherwise, we fail to reject $\mathbb{H}_0$. In the mean change detection problem in a GP, we assume that the samples are generated by a GP with zero mean under the null hypothesis, which can be stated as $\mathbb{H}_{0}: \mathbb{E}X = \boldsymbol{0}$. Under the associative alternative hypothesis at time $t$, we assume that there is a change of size $b$ in the mean as 
$\mathbb{H}_{1,t}: \exists\text{ } b \neq 0, \text{ }\mathbb{E}X = \frac{b}{2}\zeta_t$, 
where $\zeta_t \in \mathbb{R}^n$ is given by $\zeta_t(k) := sign(k-t)$ for any $t\in\mathcal{C}_{n}$. 
Combining all possible cases, the alternative hypothesis can be written as
$
\mathbb{H}_{1}: \bigcup_{t\in \mathcal{C}_{n}}\mathbb{H}_{1,t}
$,
which states that there exists at least one CP with jump size $b$. Further, we can write $2\mathfrak{L}$ as
\begin{align}
& X^T\Sigma^{-1}X
 -\min \limits_{t\in  \mathcal{C}_{n}} \min \limits_{b\ne 0}\left[\left(X{-}\frac{b}{2}\zeta_t \right)^T\Sigma^{-1}\left(X{-}\frac{b}{2}\zeta_t\right)\right],\label{eq:meanlr3}
\end{align}
where $\Sigma$ is the covariance matrix of $X$.
Maximizing Equation (\ref{eq:meanlr3}) with respect to $b$ and plugging in the test, we obtain the following test.
$$
\mathfrak{T}_{GLRT} = \mathbb{I}\left( \max_{t\in  \mathcal{C}_{n}} \left| \frac{\zeta_t^T\Sigma^{-1}X}{\sqrt{\zeta_t^T\Sigma^{-1}\zeta_t}}\right|^2 \ge \mathfrak{R}_{n,\delta} \right).
$$
With a suitable threshold $\mathfrak{R}_{n,\delta}$, the error probability is bounded as 
$
\varphi_n(\mathfrak{T}_{GLRT}) \le \delta
$ under the sufficient condition on $b$ \cite{keshavarz2015optimal}, and  
 $\mathfrak{R}_{n,\delta}$ could be chosen to be
$$
\mathfrak{R}_{n,\delta} = 1+2\left[log\left(\frac{2n}{\delta}\right)+\sqrt{log\left(\frac{2n}{\delta}\right)}\right].
$$

\section{CPD of the Covariance} \label{sec:3}
This section presents our new hypothesis tests to detect CPs in the covariance structure of a GP.


\subsection{Motivational Examples}

Figure \ref{fig:change-point} shows how CPD of a covariance structure could affect the quality of a GP regression. The left plot in Figure \ref{fig:change-point} shows samples from a GP with an intended CP in the middle. The middle plot shows samples from a GP model after the hyperparameters have been learnt using the whole datasets. The right plot shows samples from a GP model whose covariance structure breaks and the hyperparameters have been learnt separately. Figure \ref{fig:change-point} suggests that fitting nonstationary data to a time-invariant GP results in an imprecise model. GP regression with a structural break in the covariance structure is more expressive and better suited to the analysis of nonstationary data.

\subsection{Likelihood Ratio Test } \label{sec:3.2}
To construct a test for detecting changes in the covariance structure, we define the null hypothesis as $\mathbb{H}_0: \mathrm{Cov}(X_i,X_j) = K(i,j)$
and the alternative hypothesis as $\mathbb{H}_{1}=\bigcup_{t\in \mathcal{C}_{n}}\mathbb{H}_{1,t}$,
with
\begin{equation}\label{eq_hypothesis}
\mathbb{H}_{1,t}: \mathrm{Cov}(X_i,X_j) = 
\begin{cases}
K(i,j), &i,j < t \\
K'(i,j), &i,j \ge t  \\
K''(i,j), &\text{otherwise}
\end{cases}
\end{equation}
where $K$, $K'$ and $K''$ are the kernel functions.
Let $\Sigma$ and $\Sigma_t'$ denote the covariance matrices for $\mathbb{H}_0$ and $\mathbb{H}_{1,t}$, respectively. The likelihood ratio $2\mathfrak{L}$ is written as
\begin{align}
\max \limits_{t\in  \mathcal{C}_{n}} \left[ X^T(\Sigma)^{-1}X - X^T(\Sigma_{t}^{'})^{-1}X+\ln\left(\frac{|\Sigma|}{|\Sigma_{t}^{'}|}\right)\right].\label{eq:covlr}
\end{align}


\begin{theorem}\label{thm:null_distribution_case}
Let $t^*=\text{argmin}_{t}\left[X^T(\Sigma_{t}^{'})^{-1}X+\ln|\Sigma_{t}^{'}|\right]$. Under the null hypothesis, i.e., $X \sim N(\boldsymbol{0}, \Sigma)$,
\begin{align*}
2\mathfrak{L} - \ln\left(\frac{|\Sigma|}{|\Sigma_{t^*}^{'}|}\right) = \sum_{i=1}^{n}v_i-\sum_{i=1}^{n}\lambda_iu_i
\end{align*}
where $\lambda_1,...,\lambda_n$ are the eigenvalues of $\Sigma^{\frac{1}{2}}(\Sigma_{t^*}^{'})^{-1}\Sigma^{\frac{1}{2}}$ and $u_i,v_i\sim\chi_1^2$, a chi-squared distribution with degree 1.
\end{theorem}

\begin{theorem}\label{thm:alter_distribution_case}
Let $t^*=\text{argmin}_{t}\left[X^T(\Sigma_{t}^{'})^{-1}X+\ln|\Sigma_{t}^{'}|\right]$. Under the alternative hypothesis, i.e., $X \sim N(\boldsymbol{0}, {\Sigma_{t^*}^{'}})$,
\begin{align*}
2\mathfrak{L} - \ln\left(\frac{|\Sigma|}{|\Sigma_{t^*}^{'}|}\right) = \sum_{i=1}^{n}\lambda_iu_i-\sum_{i=1}^{n}v_i
\end{align*}
where $\lambda_1,...,\lambda_n$ are the eigenvalues of $\Sigma{'}_{t^*}^{\frac{1}{2}}\Sigma^{-1}\Sigma{'}_{t^*}^{\frac{1}{2}}$ and $u_i,v_i\sim\chi_1^2$, a chi-squared distribution with degree 1.
\end{theorem}
Theorems \ref{thm:null_distribution_case} and \ref{thm:alter_distribution_case} show that the difference of two positive semi-definite quadratic terms can be expressed as a subtraction between a chi-square random variable with $n$ degrees of freedom and a linear combination of independent chi-square random variables having one degree of freedom each. See Appendix\footnote{Appendix is available at \url{https://arxiv.org/abs/1905.13168}} \ref{sec:proof_sec3} for the proofs of Lemmas and Theorems in Sections \ref{sec:3} and \ref{sec:4}.

\tikzstyle{block} = [fill=gray!50, rectangle,fill opacity=0.3,
    minimum height=0.8cm, minimum width=2.8cm]
\tikzstyle{hidden} = [fill=white!20!white, rectangle, 
    minimum height=1.5cm, minimum width=0.1mm]    
\tikzstyle{sum} = [draw, fill=blue!20, circle]
\tikzstyle{input} = [coordinate]
\tikzstyle{output} = [coordinate]
\tikzstyle{pinstyle} = [pin edge={thin,white}]
\subsection{Tests for the Covariance Structural Break}\label{sec:case2}
For the case where the covariance structure breaks into two different kernels, $\mathbb{H}_0$ and $\mathbb{H}_1$ are similarly defined as in Section \ref{sec:3.2} except that $K''(i, j)=0$. The corresponding covariance matrices can then be written as 
\begin{equation*}
\Sigma = 
\begin{pmatrix}
K_{aa} & K_{ab} \\
K_{ba} & K_{bb} 
\end{pmatrix}, 
\Sigma_t^{'}=
\begin{pmatrix}
K_{aa} & 0 \\
0 & K'_{bb}
\end{pmatrix}.
\end{equation*}
Here, $K_{rc}$ for $r,c\in\{a,b\}$ indicates the covariance matrix between $X_r$ and $X_c$ with kernel $K$ where $X_a:=X_{1:t}$ and $X_b:=X_{t+1:n}$. 
We now define a likelihood ratio test as $\mathfrak{T}_{GLRT} = \mathbb{I}\left( 2\mathfrak{L} \ge \mathfrak{R}_{\delta} \right)$.
For further Lemmas and Theorems, we define constant $C_{t}$ as follows.
\begin{definition}\label{def:c1c2}
For the covariance matrices $\Sigma$ and $\Sigma_t^{'}$, 
$C_t=\frac{1}{\lambda_{\text{min}}(\Sigma)}+\frac{1}{\lambda_{\text{min}}(\Sigma_{t}')}$ where $\lambda_{\text{min}}(M)$ represents the smallest eigenvalue of matrix $M$.  
\end{definition}
\begin{lemma}\label{lemma:c01c02}
$C_{t} \le C_{0}$ for all $t\in[1,n]$ when $C_{0}=\frac{1}{\lambda_{\text{min}}(\Sigma)}+\frac{1}{\lambda_{\text{min}}(\Sigma) \wedge \lambda_{\text{min}}(\Sigma_{t=n}')}$ where $\wedge$ stands for the minimum operator. 
\end{lemma}

\begin{lemma}\label{quad_subgaussian}
Let $Q_t$ be $X^{T}((\Sigma)^{-1}- (\Sigma_{t}')^{-1})X$. When $X_t$ is bounded with $X_t \in [-V, V]$ for all $t\in[1,n]$, $Q_t - \mathbb{E}[Q_t]$ is $\frac{C_{t}V^{2}n}{2}$-subgaussian. 
\end{lemma}

Lemma \ref{quad_subgaussian} shows that, for bounded $X_t$, the tail probability of the centered quadratic term in the likelihood ratio decays approximately as fast as the Gaussian with zero mean and the same variance as the centered quadratic term.

\begin{lemma}\label{testnull2}
The probability that the $\mathfrak{T}_{GLRT}$ is correct under the null hypothesis (absence of CP) is at least $1-\delta/2$, i.e., $\mathbb{P}(2\mathfrak{L}\ge \mathfrak{R}_{n,\delta,\mathbb{H}_0}|\mathbb{H}_0) \le \delta /2,$
for $\mathfrak{R}_{n,\delta,\mathbb{H}_0} = \max_t \left(n - Tr(\Sigma(\Sigma_{t}^{'})^{-1})+\ln\left(\frac{|\Sigma|}{|\Sigma_{t}^{'}|}\right)\right) + C_{0}V^2n\;\sqrt[]{0.5\ln(2/\delta)}$.
\end{lemma}


\begin{lemma}\label{testalt2}
The probability that the $\mathfrak{T}_{GLRT}$ is correct under the alternative hypothesis (existence of a CP) is at least $1-\delta/2$, i.e., $\mathbb{P}(2\mathfrak{L}\le \mathfrak{R}_{n,\delta,\mathbb{H}_1}|\mathbb{H}_1) \le \delta /2,$
for $\mathfrak{R}_{n,\delta,\mathbb{H}_1} = \min_t \left(Tr(\Sigma_{t}^{'}(\Sigma)^{-1}) - n +\ln\left(\frac{|\Sigma|}{|\Sigma_{t}^{'}|}\right)\right) - C_{0}V^2n\;\sqrt[]{0.5\ln(2/\delta)}$.
\end{lemma}

Using Lemma \ref{quad_subgaussian} and the concentration inequality, Lemmas \ref{testnull2} and \ref{testalt2} show that we can control the type \rom{1} or type \rom{2} errors to be below $\delta / 2$. 

 

\begin{theorem}\label{thm:generalchange}
For $\mathfrak{R}_{n,\delta,\mathbb{H}_0}$, $\mathfrak{R}_{n,\delta,\mathbb{H}_1}$ in Lemmas \ref{testnull2} and \ref{testalt2}, when $\mathfrak{R}_{n,\delta,\mathbb{H}_1} \geq \mathfrak{R}_{n,\delta,\mathbb{H}_0}$ and $\mathfrak{R}_{n,\delta,\mathbb{H}_0}\le \mathfrak{R}_{\delta}\le\mathfrak{R}_{n,\delta,\mathbb{H}_1}$,
the conditional detection error probability is bounded as 
$$
\varphi_n(\mathfrak{T})=\mathbb{P}(2\mathfrak{L}\ge \mathfrak{R}_{\delta}|\mathbb{H}_0)+\max_{t\in \mathcal{C}_n}\mathbb{P}(2\mathfrak{L}\le \mathfrak{R}_{\delta}|\mathbb{H}_{1,t}) \le \delta.
$$

\end{theorem}
\begin{proof}
It follows directly from Lemmas \ref{testnull2} and \ref{testalt2}.
\end{proof}

Using Theorem \ref{thm:generalchange}, we guarantee that the likelihood ratio test for a general covariance kernel change is statistically correct for any error bound $\delta$ under specified conditions. If we set the threshold to be greater than or equal to the upper epsilon bound of the null distribution, $\mathfrak{R}_{n,\delta,\mathbb{H}_0}$, we can guarantee a bounded type \rom{1} error. If we set the threshold to be less than or equal to the lower epsilon bound of the alternative distribution, $\mathfrak{R}_{n,\delta,\mathbb{H}_1}$, we can guarantee a bounded type \rom{2} error. 

The inequalities for $\mathfrak{R}_{n,\delta,\mathbb{H}_0}$ and $\mathfrak{R}_{n,\delta,\mathbb{H}_1}$ can have three possible cases. If $\mathfrak{R}_{n,\delta,\mathbb{H}_0}>\mathfrak{R}_{n,\delta,\mathbb{H}_1}$, there is no threshold guaranteeing both type \rom{1} and type \rom{2} errors. If $\mathfrak{R}_{n,\delta,\mathbb{H}_0}=\mathfrak{R}_{n,\delta,\mathbb{H}_1}$, there is only one threshold that can guarantee both type \rom{1} and type \rom{2} errors. If $\mathfrak{R}_{n,\delta,\mathbb{H}_0}<\mathfrak{R}_{n,\delta,\mathbb{H}_1}$, the thresholds that can guarantee both type \rom{1} and type \rom{2} errors are indicated by the shaded area in Figure \ref{fig:onesidetest} in Appendix \ref{sec:proof_sec3}.


\section{Confirmatory Bayesian Online CPD} \label{sec:4}

\begin{figure*}[th!]
    \centering
      \begin{subfigure}[b]{0.49\textwidth}
        \includegraphics[width=\textwidth]{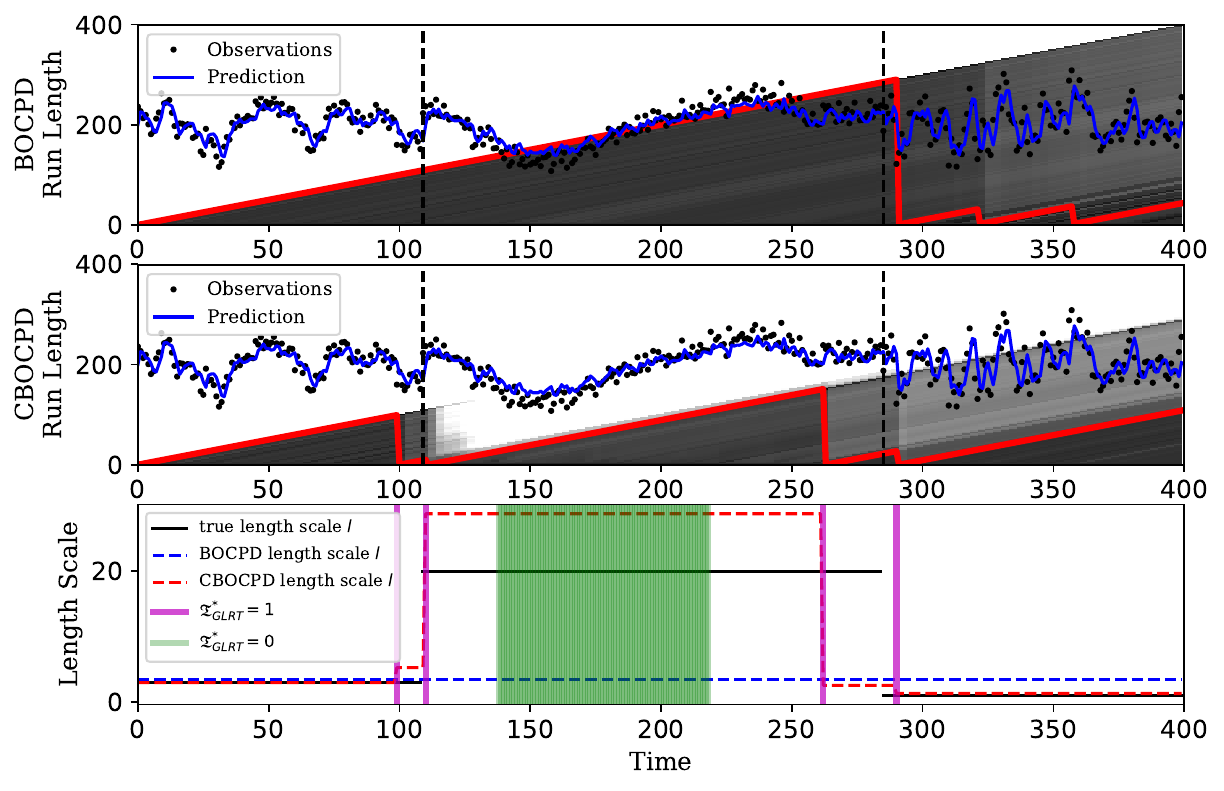}
        \caption{Timescale $\lambda=200$}
        \label{fig:synthetic_len_change_big_timescale}
     \end{subfigure}
      \begin{subfigure}[b]{0.49\textwidth}
        \includegraphics[width=\textwidth]{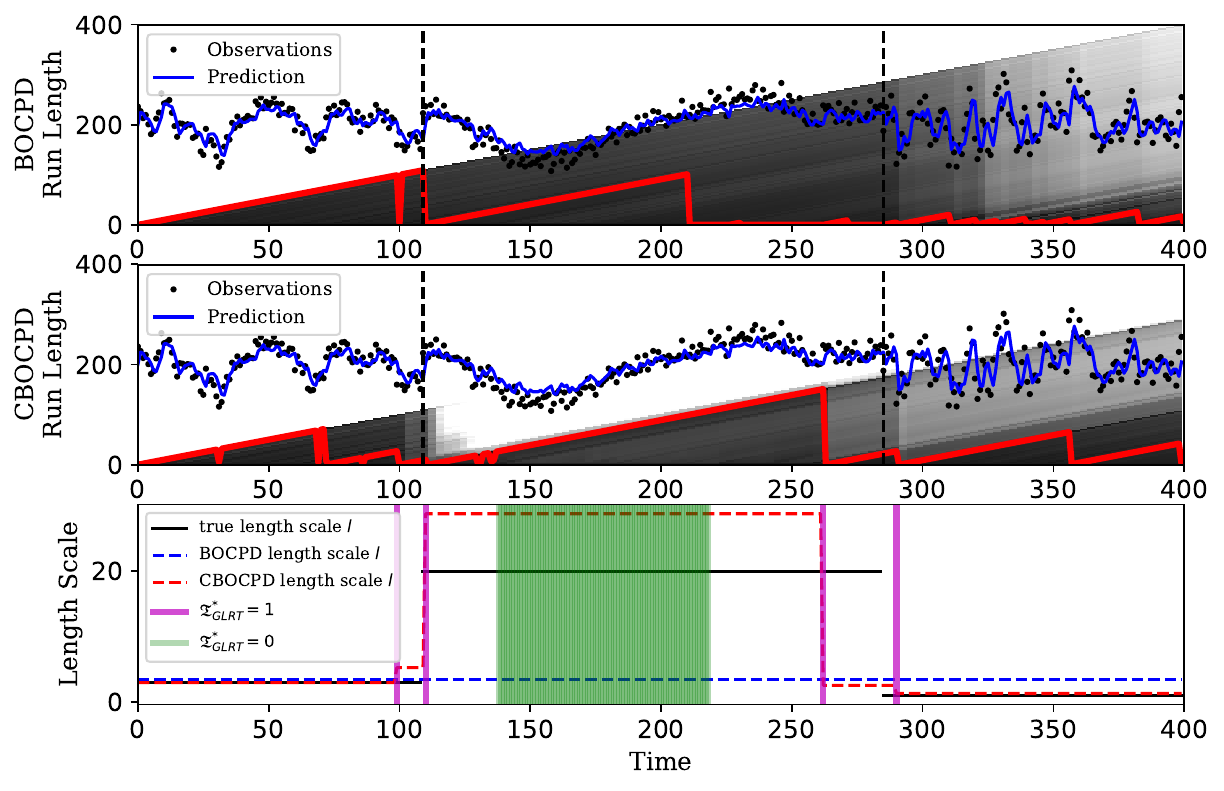}
        \caption{Timescale $\lambda=25$}
    	\label{fig:synthetic_len_change_small_timescale}
    \end{subfigure}
    \setlength{\belowcaptionskip}{-8pt}
    \caption{Synthetic data with changes in hyperparameters and the derived run-length distribution computed from BOCPD (top) and CBOCPD (middle). The bottom figure shows the hyperparameter values learned from each algorithm with the true value of the hyperparameters used to generate the synthetic data. Dashed black line indicates the true CPs.}
    \label{fig:synthetic_change}
\end{figure*}

This section presents our algorithmic contribution which improves the performance of Bayesian online CPD (BOCPD) with statistical hypothesis tests.
\subsection{Bayesian Online CPD}
Many real-world applications require CPs to be detected in an online manner.  
BOCPD \cite{BayesianOnlineChangepoint} uses Bayesian inference to update the distribution of future data based on the distribution of the run length. We assume that, if we divide data with CPs, the data in different partitions will be independent whereas the data in the same partition will be autocorrelated. 

The formulation of BOCPD is as follows. 
Let $r_t$ be the run length at time $t$ and $x_t$ be the data observation at time $t$. We denote a set of consecutive discrete observations between time $a$ and $b$ inclusive as $x_{a:b}$. Let ${x_t}^{(r_t)}$ denote the data observations since the most recent CP. We simply use ${x_t}^{(r)}$ instead of ${x_t}^{(r_t)}$ when there is no ambiguity. Then, we calculate the distribution of the next data $x_{t+1}$ given the data up to time $t$. Through marginalization of $r_t$, this can be written as
\begingroup
\small
\everymath{\scriptstyle}
\begin{align} 
\footnotesize
\mathbb{P}(x_{t{+}1}|x_{1:t})
 = \sum_{r_{t}}\mathbb{P}(x_{t{+}1}|x_{t}^{(r)})\mathbb{P}(r_{t}|x_{1:t}).
 \label{eq:bocpd1-4}
\end{align}
\normalsize



The joint distribution of the run length and data  $\mathbb{P}(r_{t},x_{1:t})$ is written as,
\begingroup
\small
\everymath{\scriptscriptstyle}
\begin{align}
 \sum_{r_{t-1}}\mathbb{P}(r_{t}|r_{t-1},x_{t-1}^{(r)})\mathbb{P}(x_{t}|r_{t-1},x_{t-1}^{(r)}) \mathbb{P}(r_{t-1},x_{1:t-1}).\label{prx4}
\end{align}
\normalsize
\endgroup


We can see that $\mathbb{P}(r_t,x_{1:t})$ has a recursive form with respect to time. The first term in Equation (\ref{prx4}) can be substituted by $\mathbb{P}(r_t|r_{t-1})$ under the assumption that the current run length only depends on the previous run length. Here, the conditional prior\footnote{$H(\tau)$ is the hazard function with $H(\tau)=\frac{\mathbb{P}_{gap}(g=\tau)}{\sum_{t=\tau}^{\infty}\mathbb{P}_{gap}(g=t)}.
$
} of $r_t$ is given by
\begin{equation}\label{eq:run_len_prob}
\mathbb{P}(r_{t}|r_{t-1}) = \begin{cases} H(r_{t-1}+1) &  r_{t} = 0\\1-H(r_{t-1}+1) &  r_{t}=r_{t-1}+1\\0 &  \text{otherwise.}\end{cases}\\[1.5mm]
\end{equation}

$\mathbb{P}_{gap} (g)$ is the a priori probability distribution over the interval between CPs. 
When $\mathbb{P}_{gap} (g)$ is a geometric distribution with timescale $\lambda$, the hazard function becomes constant as $H(\tau)=1/\lambda$. 
The second term in Equation (\ref{prx4}) can be calculated using a GP.
Thus, we can calculate the overall distribution $\mathbb{P}(x_{t+1} | x_{1:t})$ by recursive message passing scheme with $\mathbb{P}(r_t,x_{1:t})$.

The BOCPD framework efficiently determines changes in GP modeling. The change is modeled by considering all the possible run lengths. One of the main assumptions of this framework is that parameters such as the kernel parameters or the hazard function do not change. However, this online CPD framework is vulnerable to the parameters, because inappropriate parameters degrade predictive performance on non-stationary data leading imprecise run length distribution as shown in Figure \ref{fig:synthetic_change}.

\subsection{Confirmatory BOCPD}


\begin{algorithm}[tb]
   \caption{Confirmatory BOCPD }
   \label{alg:cbocpd}
\begin{algorithmic}[1]
	\STATE $m,\delta \leftarrow \text{half window size, error bound for likelihood ratio test}$
	\STATE $\mathbb{P}(x_0) \leftarrow \mathcal{N}(\mu_{prior},\sigma_{prior}^2)$  
	
    \FOR{ $t\in [1,T]$}
    \STATE $\Tilde{\mathfrak{R}}_{\delta,\mathbb{H}_0}, \Tilde{\mathfrak{R}}_{\delta,\mathbb{H}_1}\leftarrow \text{empirical thresholds for } \mathfrak{T}_{GLRT}^0,\mathfrak{T}_{GLRT}^1$
    \STATE $H \leftarrow H_{const}$
	\IF { $m<t<T-m$}
	    \STATE set window $W=x_{t-m:t+m}$
	    \STATE $2\mathfrak{L}_\tau \leftarrow \text{the likelihood ratio between } \mathbb{H}_0 \text{ and } \mathbb{H}_{1,\tau} \text{ with data }W$
	    \STATE $\tau^*, \; 2\mathfrak{L}=\text{argmax}_{\tau \in \mathcal{C}_W}2\mathfrak{L}_{\tau}, \; \text{max}_{\tau \in \mathcal{C}_W}2\mathfrak{L}_{\tau}$
	    \IF {$\mathfrak{T}_{GLRT}^1=1$ and $\mathfrak{T}_{GLRT}^0=1$ and $\tau^*=t$}
    	\STATE $H \leftarrow 1-\delta $
    	\ELSIF{$\mathfrak{T}_{GLRT}^1=0$ and $\mathfrak{T}_{GLRT}^0=0$}
    	\STATE $H \leftarrow \delta $
        \ENDIF
    \ENDIF
    
    \STATE $\pi_{t}^{(r)} \leftarrow \mathbb{P}(x_{t}|x_{t-1}^{(r)})$
    \STATE
    $\mathbb{P}(r_{t}=r_{t-1}{+}1,x_{1:t})\leftarrow$
$\mathbb{P}(r_{t-1},x_{1:t-1})\pi_{t}^{(r)}(1{-}H)$
    \STATE $\mathbb{P}(r_{t}=0,x_{1:{t-1}})\leftarrow \sum_{r_{t-1}}\mathbb{P}(r_{t-1},x_{1:t-1})\pi_t^{(r)}H$ 
    \STATE $\mathbb{P}(x_{1:{t}}) \leftarrow \sum_{r_{t}}\mathbb{P}(r_{t},x_{1:{t}})$
    \STATE $\mathbb{P}(r_{t}|x_{1:t}) \leftarrow \mathbb{P}(r_{t},x_{1:t})/\mathbb{P}(x_{1:t})$
    \STATE $\mathbb{P}(x_{t+1}|x_{1:t}) \leftarrow \sum_{r_{t}}\mathbb{P}(x_{t+1}|x_{t}^{(r)})\mathbb{P}(r_{t}|x_{1:t})$
    \ENDFOR
\end{algorithmic}
\end{algorithm}

Algorithm \ref{alg:cbocpd} presents a theoretically justified online change detection algorithm, CBOCPD. The main idea of CBOCPD is to overcome the limitations of the assumption that the run length is independent of the data. However, we claim that the first term in Equation (\ref{prx4}) can be directly calculated by the likelihood ratio test in Section \ref{sec:case2}. 
The first two lines initialize the parameters. In lines 3--13, Equation (\ref{eq:run_len_prob}) is altered as 
\begin{align*}
\mathbb{P}(r_{t}=0|r_{t-1},x_{t-1}^{(r)}) =
\begin{cases}
1-\delta, &  \tau^*=t \text{ and }\mathfrak{T}_{GLRT}^*=1 \\
\delta, & \mathfrak{T}_{GLRT}^*=0 \\
H_{const}, &\text{otherwise.}
\end{cases}
\end{align*}
Here, we have two likelihood ratio tests, $\mathfrak{T}_{GLRT}^0=\mathbb{I}\left( 2\mathfrak{L} \ge \Tilde{\mathfrak{R}}_{\delta,\mathbb{H}_0}\right)$ and $\mathfrak{T}_{GLRT}^1=\mathbb{I}\left( 2\mathfrak{L} \ge \Tilde{\mathfrak{R}}_{\delta,\mathbb{H}_1}\right)$, with thresholds in line 4. Empirical thresholds are used since the theoretically calculated thresholds in Lemmas \ref{testnull2} and \ref{testalt2} are not tight enough to use in practice. We further define $\mathfrak{T}_{GLRT}^*=1$ if $\mathfrak{T}_{GLRT}^1=1$, and $\mathfrak{T}_{GLRT}^0=1$ and $\mathfrak{T}_{GLRT}^*=0$ if $\mathfrak{T}_{GLRT}^1=0$ and $\mathfrak{T}_{GLRT}^0=0$. The likelihood ratio tests are applied to the window around $t$ with $W=x_{t-m:t+m}$. $\tau^*$ is the time point in the window that maximizes the likelihood ratio, $\tau^*=\text{argmax}_{\tau \in \mathcal{C}_W}2\mathfrak{L}_{\tau}$. Here $\mathcal{C}_W\subseteq\{t-m,...,t+m\}$ is a set of change point candidates for the window. If both likelihood ratio tests at time $\tau^*$ pass and $\tau^*$ coincides with $t$, we decide that $t$ is a CP and set $\mathbb{P}(r_{t}=0|r_{t-1},x_{t-1}^{(r)}) = 1-\delta$, which enhances the probability of change in the BOCPD framework. In contrast, if neither test passes, we strongly believe there is no change and reduce the probability of change in the BOCPD framework. This is why we name this algorithm \textit{Confirmatory} BOCPD. We add $\tau^*=t$ to avoid situations where the same time point maximizes the likelihood ratio in several consecutive windows, resulting in duplicate CPs. Lines 14--19 follow the BOCPD framework \cite{BayesianOnlineChangepoint}.

\subsection{Theoretical Analysis of CBOCPD}
We present sufficient conditions under which CBOCPD provides the lower prediction error compared to BOCPD. Here, the prediction error is defined as the expected absolute difference between the predictive mean given the true run length and the predictive mean under BOCPD and CBOCPD at a (detected) CP $t$, as a good prediction at a CP is a key factor in overall performance. For simplicity, we denote the expected value of $x_t$ under BOCPD and CBOCPD as $\mathbb{E}_{BO}[x_t|x_{1:t-1}]$ and $\mathbb{E}_{CBO}[x_t|x_{1:t-1}]$, respectively. Further we define $\alpha_i = 1-\mathbb{P}_{BO}(r_{t-1}=i|x_{1:t-1})$ under BOCPD and $\beta_i = 1-\mathbb{P}_{CBO}(r_{t-1}=i|x_{1:t-1})$ under CBOCPD.
\subsubsection{Non-stationary Case}
We first investigate the case where a change exists. 

\begin{theorem}\label{thm:main_thm_given_cp}
Consider BOCPD (Sec. 4.1) and CBOCPD (Sec. 4.2) algorithms where CBOCPD uses two statistical tests $\mathfrak{T}_{GLRT}^0$ and $\mathfrak{T}_{GLRT}^1$ with type \rom{2} error of $\delta_0^{\text{\rom{2}}}$ and $\delta_1^{\text{\rom{2}}}$ respectively. When there exists a CP at $t$ with mean value of $\mu_1$ which satisfies either
\begin{align*}
\forall i\in[0,t-1], -\epsilon_U \le \mathbb{E}[x_t|\emptyset]-\mathbb{E}[x_t|x_{i:t-1}] \le -\epsilon_L
\end{align*} 
or
\begin{align*}
\forall i\in[0,t-1], \epsilon_L \le \mathbb{E}[x_t|\emptyset]-\mathbb{E}[x_t|x_{i:t-1}] \le \epsilon_U
\end{align*} 
for $0<\epsilon_L<\epsilon_U$ where $\mathbb{E}[x_t|\emptyset]$ indicates the expected value of $x_t$ with prior distribution,
if
\begin{align*}
\frac{\epsilon_U}{\epsilon_L}\le\alpha_0\left(1+\frac{(1-\delta_0^{\text{\rom{2}}})(1-\delta_1^{\text{\rom{2}}})}{\delta_0^{\text{\rom{2}}}\delta_1^{\text{\rom{2}}}}\right)
\end{align*}
with $\alpha_0$ indicating the probability of the run length not to be zero, then the absolute error of expected value of CBOCPD at $t$ is less than or equal to the one of BOCPD as
\begin{align*}
\mathbb{E}[|\mu_1-\mathbb{E}_{BO}[x_t|x_{1:t-1}]|]\ge \mathbb{E}[|\mu_1-\mathbb{E}_{CBO}[x_t|&x_{1:t-1}]|].    
\end{align*} 

\end{theorem}

\begingroup
\footnotesize
\begin{table*}[t!]
\centering
\setlength\tabcolsep{3.5pt}
\scalebox{0.9}{ 
\begin{tabular}{c|cc|cc|cc|cc|cc}  
\toprule
\multicolumn{1}{c|}{} & \multicolumn{2}{c|}{\small{LEN-CHANGE}} & \multicolumn{2}{c|}{\small{VAR-CHANGE}} & \multicolumn{2}{c|}{\small{Gazebo:Env1}} & \multicolumn{2}{c|}{\small{Gazebo:Env2}} & \multicolumn{2}{c}{\small{Gazebo:Env3}}\\

\hline
 \multicolumn{1}{c|}{\small{Method}} & \small{NLL} & \small{MSE} & \small{NLL} & \small{MSE} & \small{NLL} & \small{MSE} & \small{NLL} & \small{MSE} & \small{NLL} & \small{MSE}\\
\hline

\midrule
\small{BOCPD} & 0.99$\pm$0.38 & 0.47$\pm$0.32 & 1.13$\pm$0.61 & 0.55$\pm$0.58 & 2.07$\pm$0.51 & 0.14$\pm$0.05 & 2.24$\pm$0.48 & 0.57$\pm$0.26 & 0.28$\pm$0.12  & 0.11$\pm$0.03  \\
\small{CS-BOCPD}      & 0.95$\pm$0.32 & 0.43$\pm$0.30& 1.05$\pm$0.42 & 0.48$\pm$0.40 & -0.11$\pm$0.33 & 0.12$\pm$0.04 & 1.82$\pm$1.51 & 0.55$\pm$0.27 & 0.25$\pm$0.16 & 0.10$\pm$0.02    \\
\hline
\small{CBOCPD}      & \textbf{0.79$\pm$0.33} & \textbf{0.34$\pm$0.24} & \textbf{0.89$\pm$0.36} & \textbf{0.41$\pm$0.34} & \textbf{-0.31$\pm$0.34} & \textbf{0.11$\pm$0.04} & \textbf{0.69$\pm$0.36} & \textbf{0.45$\pm$0.19}  &\textbf{-0.99$\pm$0.47} & 0.10$\pm$0.04   \\

\bottomrule
\end{tabular}
}
\caption{Comparison of BOCPD, CS-BOCPD, and CBOCPD with NLL and MSE on synthetic and Gazebo robot simulation datasets. }
\label{tab:synthetic}
\end{table*}
\endgroup
\raggedbottom
\subsubsection{Stationary Case}
Next, we investigate the case where a change does not exist. 

\begin{theorem}\label{thm:main_thm_wo_cp}
Consider BOCPD (Sec. 4.1) and CBOCPD (Sec. 4.2) algorithms where CBOCPD uses two statistical tests $\mathfrak{T}_{GLRT}^0$ and $\mathfrak{T}_{GLRT}^1$ with type \rom{1} error of $\delta_0^{\rom{1}}$ and $\delta_1^{\rom{1}}$ respectively. When there exists a confirmed non-CP at $t$ with mean value of $\mu_2$ which satisfies either 
\begin{gather*}
\forall i\in[0,t-1], -\epsilon_U \le \mathbb{E}[x_t|x_{1:t-1}]-\mathbb{E}[x_t|x_{i:t-1}] \le -\epsilon_L,
\end{gather*}
or
\begin{gather*}
\forall i\in[0,t-1], \epsilon_L \le \mathbb{E}[x_t|x_{1:t-1}]-\mathbb{E}[x_t|x_{i:t-1}] \le \epsilon_U
\end{gather*}
for  $0<\epsilon_L<\epsilon_U$, 
if
\begin{gather*}
\frac{\epsilon_U}{\epsilon_L}\le \frac{\alpha_{t-1}((1-\delta_0^{\rom{1}})(1-\delta_1^{\rom{1}})+\delta_0^{\rom{1}}\delta_1^{\rom{1}})}{\beta_{t-1}(1-\delta_0^{\rom{1}})(1-\delta_1^{\rom{1}})+\delta_0^{\rom{1}}\delta_1^{\rom{1}}}
\end{gather*}
with $\alpha_{t-1}$ and $\beta_{t-1}$ indicating the probability of the run length not to be $t-1$ under BOCPD and CBOCPD respectively, then the absolute error of the expected value of CBOCPD at $t$ is less than or equal to the one of BOCPD as
\begin{align*}
\mathbb{E}[|\mu_2-\mathbb{E}_{BO}[x_t|x_{1:t-1}]|]\ge \mathbb{E}[|\mu_2-\mathbb{E}_{CBO}[x_t|&x_{1:t-1}]|].    
\end{align*} 
\end{theorem}

\section{Experimental Evaluations}

\begingroup
\footnotesize
\begin{table}[h]
\centering
\setlength\tabcolsep{3.5pt}
\scalebox{0.9}{ 
\begin{tabular}{c | c c | c c}  
\toprule
\small{Method}  & \small{NLL} & \small{p-value} & \small{MSE} & \small{p-value}\\

\hline
 \multicolumn{5}{c}{Nile Data (200 training points, 463 test points)} \\
\hline

\midrule
\small{ARGP} & 1.07$\pm$0.64 & $<$0.0001 & 5.06$\pm$0.86 & 0.0005   \\
\small{ARGP-BOCPD}      & 0.78$\pm$0.72 & $<$0.0001 & 4.94$\pm$0.87 & 0.0017     \\
\small{GPTS} & 0.86$\pm$0.64 & $<$0.0001 & 4.78$\pm$0.81 & 0.0100      \\
\small{BOCPD} & 0.57$\pm$0.77 & 0.0014 & 4.73$\pm$0.82 & 0.0115    \\
\small{CBOCPD}      & \textbf{0.00$\pm$0.80} & N/A & \textbf{4.32$\pm$0.74}  & N/A      \\

\hline
 \multicolumn{5}{c}{Well Log Data (500 training points, 3050 test points)} \\
\hline

\small{ARGP} & 7.20$\pm$0.60 & $<$0.0001 & 17.3$\pm$2.6 & $<$0.0001 \\
\small{ARGP-BOCPD}       & \textbf{0.00$\pm$0.30} & N/A & \textbf{4.68$\pm$0.46} & N/A \\
\small{GPTS} & 3.73$\pm$0.42 & $<$0.0001 & 8.27$\pm$0.61 & $<$0.0001     \\
\small{BOCPD}  & 4.35$\pm$0.31 & $<$0.0001 & 19.2$\pm$1.3 & $<$0.0001  \\
\small{CBOCPD}      & 0.30$\pm$0.27 & 0.0010 & 4.92$\pm$0.44 & 0.2124    \\

\hline
 \multicolumn{5}{c}{Snow Data (500 training points, 13380 test points)} \\
\hline

\small{ARGP} & 17.48$\pm$0.82 & $<$0.0001 & 14.82$\pm$0.57 & $<$0.0001  \\
\small{ARGP-BOCPD} & 0.06$\pm$0.39 & $<$0.0001 & 9.65$\pm$0.39 & $<$0.0001       \\
\small{GPTS} & 16.60$\pm$0.22 & $<$0.0001 & 8.76$\pm$0.36 & $<$0.0001  \\
\small{BOCPD}  & \textbf{0.00$\pm$0.39} & N/A & 9.43$\pm$0.38 & N/A \\
\small{CBOCPD}  & 1.92$\pm$0.37 & $<$0.0001 &  \textbf{6.34$\pm$0.27}  & $<$0.0001       \\

\bottomrule
\end{tabular}
}
\caption{Results of predictive performance on Nile data, Well Log Data and Snow Data. The results are provided with 95\% of confidence interval and the p-value 
of the null hypothesis that a method is equivalent to the best performing method according to NLL, using a one sided t-test.}
\label{tab:real}
\end{table}
\endgroup
\raggedbottom

\subsection{Synthetic Data}
In this experiment, we investigate two synthetic datasets generated by GPs with changes in the length scale and the variance of a Radial Basis Function kernel, respectively. Observations are obtained by adding the white Gaussian noise with the variance $\sigma_{no}^2=0.1$. For both datasets, two CPs are drawn uniformly from time intervals $(75, 125)$ and $(275, 325)$ with end time $T=400$.

For the first experiment, the length scale $l$ switches from $3$ to $20$ and from $20$ to $1$ with the variance $\sigma^2=1$. For the second experiment, the variance $\sigma^2$ switches from $1$ to $4$ and from $4$ to $0.3$ with the length scale $l=3$. We compare the proposed CBOCPD with BOCPD and CS (CUSUM)-BOCPD. For BOCPD, we use a modified version of the stationary GP-BOCPD \cite{GP-BOCPD} by fixing the timescale $\lambda$. CS-BOCPD uses CUSUM instead of the our proposed likelihood ratio test in Algorithm \ref{alg:cbocpd}. For all algorithms, the timescale of $\lambda=200$ is used.

Figure \ref{fig:synthetic_len_change_big_timescale} and Figure \ref{fig:synthetic_len_change_small_timescale} show that CBOCPD identifies the length scale change in the data with the help of statistical tests, whereas BOCPD captures the change too less or too many times. The first two columns in Table \ref{tab:synthetic} present the quantitative results from 100 runs of each algorithm after training on the first 100 points and testing on the remaining 300 points. CBOCPD outperforms BOCPD in terms of both negative log-likelihood (NLL) and mean squared error (MSE). The differences in performances become larger when there is a stronger correlation among the samples. 


\subsection{Gazebo Robot Simulation Data}
We also conducted experiments on the Gazebo robot simulator to detect changes in the environment of a robot. We gathered data by moving the robot through a changing environment (i.e., varying ground properties). We used the Pioneer3AT robot with kinetic ROS. There are three environments. In the first environment (Env1), the ground changes at the midpoint from `Plane ground' to `Bumpy ground 1'. In the second environment (Env2), the robot moves from `Bumpy ground 1' to `Bumpy ground 2' where `Bumpy ground 2' is more coarser than `Bumpy ground 1'. Finally, in the third environment (Env3), the robot moves from `Bumpy ground 2' back to the `Plane ground'. For simplicity, we used only z-directional position data as the inputs. The last three columns in Table \ref{tab:synthetic} show that CBOCPD outperforms other methods in all environments. Interestingly, the results show that CUSUM test does not help BOCPD much when the variance decreases as in Env3 but proposed likelihood ratio test improves BOCPD in all the cases.

\subsection{Real World Data}
Further, we compare our proposed algorithm with autoregressive GP (ARGP), autoregressive GP with BOCPD (ARGP-BOCPD), GP time series with BOCPD (BOCPD) and GP time series model (GPTS). For baseline methods we used the publicly available source codes\footnote{\url{https://sites.google.com/site/wwwturnercomputingcom/software/ThesisCodeAndData.zip}}.

We first consider Nile data which records the lowest annual water levels on the Nile river during AD 622-1284 measured at the island of Roda. There is domain knowledge suggesting a CP in the year 715 due to the construction of a `nilometer' \cite{popper1951cairo,balek1977hydrology}. 
Secondly, we used Well Log data which records 4050 nuclear magnetic resonance measurements while drilling a well.
We also considered Snowfall data which records daily snowfall data in Whistler, BC, Canada. In this dataset each method is evaluated by the ability to predict snowfall of next day using 35 years of test data.

Table \ref{tab:real} shows the predictive performance of our proposed algorithm compared to other GP based BOCPD methods. In Nile data, we see that combining BOCPD method with GP improves the performance. CBOCPD further improves the performance by leveraging the proposed statistical hypothesis tests and outperforms all other algorithms. In Well Log data, the nonlinear temporal correlations within each regime give a slight advantage to underlying predictive model of ARGP. However, CBOCPD still shows the competitive result. In Snowfall data, CBOCPD well detects the difference in noise levels and achieves the best performance in MSE.

\section{Conclusion}
This paper presents a novel framework for detecting changes in the covariance structure of the GP. We propose a semi-online CPD algorithm, Confirmatory BOCPD, which is an improved version of BOCPD with statistical hypothesis tests. Experiments using synthetic and real-world datasets demonstrate that CBOCPD outperforms conventional BOCPDs.  

\section*{Acknowledgments}
 This work was supported by the Industrial Convergence Core Technology Development Program funded by the Ministry of Trade, Industry \& Energy (MOTIE), Korea (No. 10063172, Development of robot intelligence technology for mobility with learning capability toward robust and seamless indoor and outdoor autonomous navigation) and Institute for Information \& communications Technology Planning \& Evaluation (IITP) grant funded by the Ministry of Science and ICT (MSIT), Korea (No. 2017-0-01779, A machine learning and statistical inference framework for explainable artificial intelligence).

\begingroup
\footnotesize
\bibliographystyle{named}
\bibliography{winnower_template}
\endgroup


\clearpage

\appendix
\onecolumn

\section{Proofs for Sections 3 and 4}\label{sec:proof_sec3}

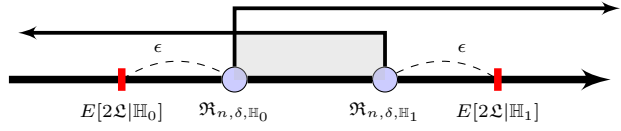
\begin{figure}[h]
\centering

\begin{tikzpicture}[auto, node distance=2cm,>=latex']
    \node [input, name=input] {};
    \node [sum, right of=input, xshift=1cm,label={below:$\mathfrak{R}_{n,\delta,\mathbb{H}_0}$}] (sum) {};
    \node [sum, right of=sum, label={below:$\mathfrak{R}_{n,\delta,\mathbb{H}_1}$}] (controller) {};
    \node [output, right of=controller, xshift=1cm] (output) {};
    \node [hidden, above of=output, yshift=-1.2cm] (output2) {};
    \node [hidden, above of=input, yshift=-1.4cm] (output3) {};
    \node [block, right of=sum,yshift=0.3cm, xshift=-1cm, minimum width=2cm, minimum height=0.6cm] (shade) {};
    
    

    \draw [draw,-, line width=1mm] (input) --  (sum);
    
    \draw [-,line width=1mm] (sum) -- (controller);
    \draw [->,line width=1mm] (controller) -- node [name=y] {}(output);
    
    \draw [style=dashed] (1.5,0) to [bend left] node {$\epsilon$}(sum);
	\draw [style=dashed] (controller) to [bend left] node {$\epsilon$}(6.5,0);

	\draw [->, line width=0.5mm] (sum) |- node[pos=0.99] {}
        node [near end] {} (output2.50);
	\draw [->,line width=0.5mm] (controller) |- node[pos=0.3] {}
        node [near end] {} (output3.10);

    \draw [thick, red,line width=1mm] (1.5,0.15)  -- (1.5,-0.15) node[below, black] {$ E[2\mathfrak{L}|\mathbb{H}_0]$};
    \draw [thick, red,line width=1mm] (6.5,0.15)  -- (6.5,-0.15) node[below, black] {$ E[2\mathfrak{L}|\mathbb{H}_1]$};
    
\end{tikzpicture}
\caption{Range of thresholds that guarantee bounded type \rom{1} (arrow pointing right) and type \rom{2} (arrow pointing left) errors.}
\label{fig:onesidetest}
\end{figure}


\begin{definition}[Subgaussianity]\label{subgaussian}
A random variable $X$ is $\sigma$-subgaussian if for all $\lambda\in\mathbb{R}$ it holds that $\mathbb{E}[\exp(\lambda X)] \le \exp(\lambda^{2}\sigma^{2}/2)$.
\end{definition}

\begin{lemma}\label{bdd_subgaussian}
If $X$ is a bounded and centered random variable, with $X \in [a,b]$, then $X$ is $\frac{b-a}{2}$-subgaussian.
\end{lemma}
Please refer to Ch. 5.2 in \cite{lattimore2018bandit}.

\begin{theorem}[Chernoff Bound]\label{chernoff bound}
If $X$ is $\sigma$-subgaussian, then for any $\epsilon \ge 0$,
\begin{align*}
\mathbb{P}(X \ge \epsilon) \le \exp(-\frac{\epsilon^2}{2\sigma^2}), \;\; \text{and } \;\; \mathbb{P}(X \le -\epsilon) \le \exp(-\frac{\epsilon^2}{2\sigma^2}).
\end{align*}
\end{theorem}

\begin{theorem}[Cauchy Interlace Theorem]\label{cauchy_interlace_theorem}
Let $A$ be a Hermitian matrix of order $n$, and let $B$ be a principal submatrix of $A$ of order $n-1$, and denote the $k$th largest eigenvalues of a general $n \times n$ matrix by $\lambda_k(\Sigma)$ so that
\begin{align*}
\lambda_1(\Sigma) \ge \lambda_2(\Sigma) \ge \cdots \ge \lambda_n(\Sigma), \; \forall k \in [1, n].
\end{align*}
If $\lambda_n(A) \le \lambda_{n-1}(A) \le \dots \le \lambda_2(A) \le \lambda_1(A)$ and $\lambda_{n-1}(B) \le \lambda_{n-2}(B) \le \dots \le \lambda_2(B) \le \lambda_1(B)$, then $\lambda_n(A) \le \lambda_{n-1}(B) \le \lambda_{n-1}(A) \le \lambda_{n-2}(B) \le \cdots \le \lambda_{2}(A) \le \lambda_{1}(B) \le \lambda_{1}(A)$.
\end{theorem}

\begin{corollary}\label{ineq_eigv_submat}
Let $A, B, C_t \in \mathbb{R}^{n \times n}$ be positive semi-definite matrix defined element wise as $A_{i,j}=K(i, j)$, $B_{i,j}=K'(i, j)$ and $C_t$ is defined as 
\begin{align*}
C_{t,(i,j)}=\begin{cases}
K(i,j), &i,j < t  \\
K'(i,j), &i,j \ge t  \\
0, &\text{otherwise}
\end{cases}
\end{align*}
for all $t\in[1, n]$. Then, $\forall 1 \le k \le n, (\lambda_n(A) \wedge \lambda_n(B)) \le \lambda_k(C_t) \le (\lambda_1(A) \vee \lambda_1(B))$, where $\wedge$ and $\vee$ stand for minimum and maximum operators.

\end{corollary}




\begin{proof}[Proof of Lemma \ref{lemma:c01c02}]
The proof follows from Corollary \ref{ineq_eigv_submat}.
\end{proof}

\begin{proof}[Proof of Theorem \ref{thm:null_distribution_case}]
The quadratic forms in $2\mathfrak{L}$ can be reorganized as 
\begin{align*}
X^T(\Sigma)^{-1}X - X^T(\Sigma_{t^*}^{'})^{-1}X = X^T\Sigma^{-\frac{1}{2}}\Sigma^{\frac{1}{2}}\Sigma^{-1}\Sigma^{\frac{1}{2}}\Sigma^{-\frac{1}{2}}X - X^T\Sigma^{-\frac{1}{2}}\Sigma^{\frac{1}{2}}(\Sigma_{t^*}^{'})^{-1}\Sigma^{\frac{1}{2}}\Sigma^{-\frac{1}{2}}X
\end{align*}
Letting $X'=\Sigma^{-\frac{1}{2}}X$, one has $X' \sim N(\boldsymbol{0},I)$ since $X \sim N(\boldsymbol{0},\Sigma)$, and then it gives:
\begin{align*}
{X'}^{T}X' - {X'}^T\Sigma^{\frac{1}{2}}(\Sigma_{t^*}^{'})^{-1}\Sigma^{\frac{1}{2}}X'
\end{align*}
The matrix $\Sigma^{\frac{1}{2}}(\Sigma_{t^*}^{'})^{-1}\Sigma^{\frac{1}{2}}$ is symmetric as both matrices $\Sigma$ and $\Sigma_{t^*}^{'}$ are symmetric. Performing eigenvalue decomposition gives $\Sigma^{\frac{1}{2}}(\Sigma_{t^*}^{'})^{-1}\Sigma^{\frac{1}{2}}=Q\Lambda Q^T$ where $Q$ is an orthogonal matrix and diagonal elements of $\Lambda$ is non-negative since $\Sigma$ and $\Sigma_{t^*}^{'}$ are positive semi-definite matrices. Letting  $X''=Q^TX'$, one has $X'' \sim N(\boldsymbol{0},I)$ since $Q$ is an orthogonal matrix, and then the equation is further transformed as:
\begin{align*}
{X'}^{T}X' - {X''}^T\Lambda{X''} = \sum_{i=1}^{n}v_i-\sum_{i=1}^{n}\lambda_iu_i
\end{align*}
where $\lambda_1,...,\lambda_n$ are the eigenvalues of $\Lambda$ and $u_i,v_i\sim\chi_1^2$, a chi-squared distribution with degree 1.
The last equality holds as Q is orthonormal.

\end{proof}

\begin{proof}[Proof of Theorem \ref{thm:alter_distribution_case}]
The proof proceeds in a similar manner as Theorem \ref{thm:null_distribution_case} except for that $X \sim N(\boldsymbol{0},{\Sigma_{t^*}^{'}})$. The quadratic forms in $2\mathfrak{L}$ can be reorganized as 
\begin{align*}
X^T(\Sigma)^{-1}X - X^T(\Sigma_{t^*}^{'})^{-1}X = X^T\Sigma{'}_{t*}^{-\frac{1}{2}}\Sigma{'}_{t*}^{\frac{1}{2}}\Sigma^{-1}\Sigma{'}_{t*}^{\frac{1}{2}}\Sigma{'}_{t*}^{-\frac{1}{2}}X - X^T\Sigma{'}_{t*}^{-\frac{1}{2}}\Sigma{'}_{t*}^{\frac{1}{2}}(\Sigma_{t^*}^{'})^{-1}\Sigma{'}_{t*}^{\frac{1}{2}}\Sigma{'}_{t*}^{-\frac{1}{2}}X
\end{align*}
Letting $X'=\Sigma{'}_{t*}^{-\frac{1}{2}}X$, one has $X' \sim N(\boldsymbol{0},I)$ since $X \sim N(\boldsymbol{0},\Sigma{'}_{t*})$, and then it gives:
\begin{align*}
{X'}^T\Sigma{'}_{t*}^{\frac{1}{2}}\Sigma^{-1}\Sigma{'}_{t*}^{\frac{1}{2}}X' - {X'}^{T}X'argmaxtZt
\end{align*}
The matrix $\Sigma{'}_{t*}^{\frac{1}{2}}\Sigma^{-1}\Sigma{'}_{t*}^{\frac{1}{2}}$ is symmetric as both matrices $\Sigma$ and $\Sigma_{t^*}^{'}$ are symmetric. Performing eigenvalue decomposition gives $\Sigma{'}_{t*}^{\frac{1}{2}}\Sigma^{-1}\Sigma{'}_{t*}^{\frac{1}{2}}=Q\Lambda Q^T$ where $Q$ is an orthogonal matrix and diagonal elements of $\Lambda$ is non-negative since $\Sigma$ and $\Sigma_{t^*}^{'}$ are positive semi-definite matrices. Letting  $X''=Q^TX'$, one has $X'' \sim N(\boldsymbol{0},I)$ since $Q$ is an orthogonal matrix, and then the equation is further transformed as:
\begin{align*}
{X''}^T\Lambda{X''} - {X'}^{T}X' = \sum_{i=1}^{n}\lambda_iu_i-\sum_{i=1}^{n}v_i
\end{align*}
where $\lambda_1,...,\lambda_n$ are the eigenvalues of $\Lambda$ and $u_i,v_i\sim\chi_1^2$, a chi-squared distribution with degree 1.

\end{proof}

\begin{proof}[Proof of Lemma \ref{quad_subgaussian}]
We will first show the bound of a quadratic term, $X^T(\Sigma^{-1}-(\Sigma_t')^{-1})X$. For a positive semi-definite matrix $M$, its eigenvectors are orthonormal and all of its eigenvalues are non-negative. By performing eigenvalue decomposition
\begin{align*}
X^{T}M X = X^{T}Q\Lambda Q^{T}X = \sum_{k} \lambda_{k}{<X, q_{k}>}^{2},
\end{align*}
where $\lambda_k$ is the $k$th diagonal element of $\Lambda$ and $q_k$ is the $k$th column of $Q$. Since eigenvectors are orthonormal
\begin{align*}
\sum_{k}{<X, q_{k}>}^{2} = X^{T}QQ^{T}X=<X,X>=||X||^{2}. 
\end{align*}
By the restriction that $X_t$ is bounded with $X_t \in [-V, V]$ for all $t$, 
\begin{align*}
0 \le X^{T}M X \le \lambda_{\text{max}}(M)V^{2}n,
\end{align*}
where $\lambda_{\text{min}}(M)$ and $\lambda_{\text{max}}(M)$ represent the smallest and biggest eigenvalues of $M$ respectively. 

As $\Sigma$ is symmetric, its inverse $(\Sigma)^{-1}$ is also symmetric and positive semi-definite. Substituting $M$ by $\Sigma^{-1}$ gives
\begin{align*}
0 \le X^{T}(\Sigma)^{-1} X \le \lambda_\text{max}(\Sigma^{-1})V^2n \le \frac{1}{\lambda_{\text{min}}(\Sigma)}V^{2}n,
\end{align*}
because eigenvalues of the inverse matrix are inverses of eigenvalues of an original matrix. 

Similarly for $(\Sigma_{t}')^{-1}$,
\begin{align*}
0 \le X^{T}(\Sigma_{t}')^{-1} X \le \lambda_{\text{max}}((\Sigma_{t}')^{-1})V^2n  \le \frac{1}{\lambda_{\text{min}}(\Sigma_{t}')}V^{2}n.
\end{align*}
Then the difference of the quadratic term of our interest is bounded as
\begin{align*}
-\frac{1}{\lambda_{\text{min}}(\Sigma_{t}')}V^{2}n \le X^{T}((\Sigma)^{-1}- (\Sigma_{t}')^{-1})X \le \frac{1}{\lambda_{\text{min}}(\Sigma)}V^{2}n,
\end{align*}
with $C_{t}=\frac{1}{\lambda_{\text{min}}(\Sigma)}+\frac{1}{\lambda_{\text{min}}(\Sigma_{t}')}$. Thus, we can conclude proof by Lemma \ref{bdd_subgaussian}.
\end{proof}

\begin{proof}[Proof of Lemma \ref{testnull2}]
For brevity, let $Z_t=X^T(\Sigma^{-1}-{\Sigma'_{t}}^{-1})X + \ln\left(\frac{|\Sigma|}{|\Sigma_{t}^{'}|}\right),\; t=1,...,n$. Lemma \ref{quad_subgaussian} implies that $Z_t - \mathbb{E}[Z_t]$ is $\frac{C_{t}V^2n}{2}$-subgaussian. Under the null hypothesis the expectation of $Z_t$ is defined as 

\begin{align*}
\mathbb{E}(Z_t|\mathbb{H}_0) &= \mathbb{E}\left[ X^T(\Sigma)^{-1}X - X^T(\Sigma_{t}^{'})^{-1}X+\ln\left(\frac{|\Sigma|}{|\Sigma_{t}^{'}|}\right)\middle|\mathbb{H}_0\right]\\
&=\mathbb{E}\left[ Tr(XX^T(\Sigma)^{-1}) - Tr(XX^T(\Sigma_{t}^{'})^{-1})+\ln\left(\frac{|\Sigma|}{|\Sigma_{t}^{'}|}\right)\middle|\mathbb{H}_0\right]\\
&= Tr(\Sigma(\Sigma)^{-1}) - Tr(\Sigma(\Sigma_{t}^{'})^{-1})+\ln\left(\frac{|\Sigma|}{|\Sigma_{t}^{'}|}\right)\\
&= n - Tr(\Sigma(\Sigma_{t}^{'})^{-1})+\ln\left(\frac{|\Sigma|}{|\Sigma_{t}^{'}|}\right).
\end{align*} 

Theorem \ref{chernoff bound} implies that $\mathbb{P}(Z_t \ge n - Tr(\Sigma(\Sigma_{t}^{'})^{-1})+\ln\left(\frac{|\Sigma|}{|\Sigma_{t}^{'}|}\right) + C_{t}V^2n\;\sqrt[]{0.5\ln(2n/\delta)}) \le \frac{\delta}{2n}$. By Lemma \ref{lemma:c01c02}, $C_{t} \le C_{0}$ for all $t\in[1,n]$. Thus, 
\begin{align*}
\mathbb{P}\left[Z_t \ge \max_t \left(n - Tr(\Sigma(\Sigma_{t}^{'})^{-1})+\ln\left(\frac{|\Sigma|}{|\Sigma_{t}^{'}|}\right)\right) + C_{0}V^2n\;\sqrt[]{0.5\ln(2/\delta)}\right] \le \frac{\delta}{2}.
\end{align*} 
As the above inequality is satisfied for all $t\in[1,n]$, it is also satisfied for $t*=\text{argmax}_t Z_t$. Thus we can conclude as follows.
\begin{align*}
\mathbb{P}\left[\max_t Z_t \ge \max_t \left(n - Tr(\Sigma(\Sigma_{t}^{'})^{-1})+\ln\left(\frac{|\Sigma|}{|\Sigma_{t}^{'}|}\right)\right) + C_{0}V^2n\;\sqrt[]{0.5\ln(2n/\delta)}\right] \le \frac{\delta}{2}.
\end{align*} 
\end{proof}

\begin{proof}[Proof of Lemma \ref{testalt2}]
For brevity, let $Z_t=X^T(\Sigma^{-1}-{\Sigma'_{t}}^{-1})X + \ln\left(\frac{|\Sigma|}{|\Sigma_{t}^{'}|}\right),\; t=1,...,n$. Lemma \ref{quad_subgaussian} implies that $Z_t - \mathbb{E}[Z_t]$ is $\frac{C_{t}V^2n}{2}$-subgaussian. Under the alternative hypothesis the expectation of $Z_t$ is defined as 

\begin{align*}
\mathbb{E}(Z_t|\mathbb{H}_{1,t}) &= \mathbb{E}\left[ X^T(\Sigma)^{-1}X - X^T(\Sigma_{t}^{'})^{-1}X+\ln\left(\frac{|\Sigma|}{|\Sigma_{t}^{'}|}\right)\middle|\mathbb{H}_{1,t}\right]\\
&=\mathbb{E}\left[ Tr(XX^T(\Sigma)^{-1}) - Tr(XX^T(\Sigma_{t}^{'})^{-1})+\ln\left(\frac{|\Sigma|}{|\Sigma_{t}^{'}|}\right)\middle|\mathbb{H}_{1,t}\right]\\
&= Tr(\Sigma'_{t}(\Sigma)^{-1}) - Tr(\Sigma'_{t}(\Sigma_{t}^{'})^{-1})+\ln\left(\frac{|\Sigma|}{|\Sigma_{t}^{'}|}\right)\\
&= Tr(\Sigma'_{t}(\Sigma)^{-1}) - n +\ln\left(\frac{|\Sigma|}{|\Sigma_{t}^{'}|}\right).
\end{align*} 

Theorem \ref{chernoff bound} implies that $\mathbb{P}(Z_t \le Tr(\Sigma_{t}^{'}(\Sigma)^{-1}) - n +\ln\left(\frac{|\Sigma|}{|\Sigma_{t}^{'}|}\right) - C_{t}V^2n\;\sqrt[]{0.5\ln(2/\delta)}) \le \frac{\delta}{2}$. By Lemma \ref{lemma:c01c02}, $C_{t} \le C_{0}$ for all $t\in[1,n]$. Thus, 
\begin{align*}
\mathbb{P}\left[Z_t \le \min_t \left(Tr(\Sigma_{t}^{'}(\Sigma)^{-1}) - n +\ln\left(\frac{|\Sigma|}{|\Sigma_{t}^{'}|}\right)\right) - C_{0}V^2n\;\sqrt[]{0.5\ln(2/\delta)}\right] \le \frac{\delta}{2}.
\end{align*} 
As the above inequality is satisfied for all $t\in[1,n]$, it is also satisfied for $t*=\text{argmax}_t Z_t$. Thus we can conclude as follows.
\begin{align*}
\mathbb{P}\left[\max_t Z_t \le \min_t \left(Tr(\Sigma_{t}^{'}(\Sigma)^{-1}) - n +\ln\left(\frac{|\Sigma|}{|\Sigma_{t}^{'}|}\right)\right) - C_{0}V^2n\;\sqrt[]{0.5\ln(2/\delta)}\right] \le \frac{\delta}{2}.
\end{align*}
\end{proof}

\begin{proof}[Proof of Theorem \ref{thm:main_thm_given_cp}]
Let's define the gain of CBOCPD over BOCPD as 
$$
|\mathbb{E}[x_t|\emptyset]-\mathbb{E}_{BO}[x_t|x_{1:t-1}]|-|\mathbb{E}[x_t|\emptyset]-\mathbb{E}_{CBO}[x_t|x_{1:t-1}]|.
$$
In the case $\mathfrak{T}_{GLRT}^*=1$, the gain is written as 
\begin{align*}
    &|\mathbb{E}[x_t|\emptyset]-\mathbb{E}_{BO}[x_t|x_{1:t-1}]|-|\mathbb{E}[x_t|\emptyset]-\mathbb{E}[x_t|\emptyset]|\\
    &\quad=|\mathbb{E}[x_t|\emptyset]-\mathbb{E}_{BO}[x_t|x_{1:t-1}]|\\
    &\quad=\left\vert\sum_{r_{t-1}=0}^{t-1}\mathbb{E}[x_t|\emptyset]\mathbb{P}_{BO}(r_{t-1}|x_{1:t-1})-\sum_{r_{t-1}=0}^{t-1}\mathbb{E}[x_t|x_{t-1}^{(r)}]\mathbb{P}_{BO}(r_{t-1}|x_{1:t-1})\right\vert\\
    &\quad=\left\vert\sum_{r_{t-1}=1}^{t-1}\left(\mathbb{E}[x_t|\emptyset]-\mathbb{E}[x_t|x_{t-1}^{(r)}]\right)\mathbb{P}_{BO}(r_{t-1}|x_{1:t-1})\right\vert\\
    &\quad\ge\left\vert\epsilon_L\sum_{r_{t-1}=1}^{t-1}\mathbb{P}_{BO}(r_{t-1}|x_{1:t-1})\right\vert=\epsilon_L\cdot\alpha_0.
\end{align*}
The inequality in the last line comes from the assumption $\epsilon_L \le |\mathbb{E}[x_t|\emptyset]-\mathbb{E}[x_t|x_{i:t-1}]| \le \epsilon_U$.

In the case $\mathfrak{T}_{GLRT}^*=0$, the loss of CBOCPD is written as 
\begin{align*}
    &|\mathbb{E}[x_t|\emptyset]-\sum_{r_{t-1}=1}^{t-1}\mathbb{E}[x_t|x_{t-1}^{(r)}]\cdot \mathbb{P}_{CBO}(r_{t-1}|x_{1:t-1})|\\
    &\quad=\left\vert\sum_{r_{t-1}=1}^{t-1}\mathbb{E}[x_t|\emptyset]\mathbb{P}_{CBO}(r_{t-1}|x_{1:t-1})-\sum_{r_{t-1}=1}^{t-1}\mathbb{E}[x_t|x_{t-1}^{(r)}]\mathbb{P}_{CBO}(r_{t-1}|x_{1:t-1})\right\vert\\
    &\quad=\left\vert\sum_{r_{t-1}=1}^{t-1}\left(\mathbb{E}[x_t|\emptyset]-\mathbb{E}[x_t|x_{t-1}^{(r)}]\right)\mathbb{P}_{CBO}(r_{t-1}|x_{1:t-1})\right\vert\\
    &\quad\le\left\vert\epsilon_U\sum_{r_{t-1}=1}^{t-1}\mathbb{P}_{CBO}(r_{t-1}|x_{1:t-1})\right\vert=\epsilon_U.
\end{align*}
The equation in the last line comes from the fact that $\Sigma_{r_{t-1}}^{t-1}\mathbb{P}_{CBO}(r_{t-1}|x_{1:t-1})=1$ under the CBOCPD when non-change is detected. 
Then, the gain is bounded as
\begin{align*}
    &|\mathbb{E}[x_t|\emptyset]-\mathbb{E}_{BO}[x_t|x_{1:t-1}]|-|\mathbb{E}[x_t|\emptyset]-\mathbb{E}_{CBO}[x_t|x_{1:t-1}]| \ge \epsilon_L\cdot\alpha_0-\epsilon_U.
\end{align*}
As $\mathbb{P}(\mathfrak{T}_{GLRT}^*=1)=(1-\delta_0^{\text{\rom{2}}})(1-\delta_1^{\text{\rom{2}}})$ and $\mathbb{P}(\mathfrak{T}_{GLRT}^*=0)=\delta_0^{\text{\rom{2}}}\delta_1^{\text{\rom{2}}}$ in non-stationary case, the expected gain is bounded from below as
\begin{align*}
&\mathbb{E}(|\mathbb{E}[x_t|\emptyset]-\mathbb{E}_{BO}[x_t|x_{1:t-1}]|-|\mathbb{E}[x_t|\emptyset]-\mathbb{E}_{CBO}[x_t|x_{1:t-1}]|)\\
&\quad\ge\epsilon_L\alpha_0(1-\delta_0^{\text{\rom{2}}})(1-\delta_1^{\text{\rom{2}}})+(\epsilon_L\alpha_0-\epsilon_U)\delta_0^{\text{\rom{2}}}\delta_1^{\text{\rom{2}}}\ge0
\end{align*}
where the last inequality follows the assumption. Thus we can conclude that the expected gain is non-negative.

\end{proof}

\begin{proof}[Proof of Theorem \ref{thm:main_thm_wo_cp}]
Let's define the gain of CBOCPD over BOCPD as 
$$
|\mathbb{E}[x_t|x_{1:t-1}]-\mathbb{E}_{BO}[x_t|x_{1:t-1}]|-|\mathbb{E}[x_t|x_{1:t-1}]-\mathbb{E}_{CBO}[x_t|x_{1:t-1}]|.
$$
The loss of BOCPD is written as 
\begin{align*}
    &|\mathbb{E}[x_t|x_{1:t-1}]-\mathbb{E}_{BO}[x_t|x_{1:t-1}]|\\
    &\quad=\left\vert\sum_{r_{t-1}=0}^{t-1}\mathbb{E}[x_t|x_{1:t-1}]\mathbb{P}_{BO}(r_{t-1}|x_{1:t-1})-\sum_{r_{t-1}=0}^{t-1}\mathbb{E}[x_t|x_{t-1}^{(r)}]\mathbb{P}_{BO}(r_{t-1}|x_{1:t-1})\right\vert\\
    &\quad=\left\vert\sum_{r_{t-1}=0}^{t-2}\mathbb{E}[x_t|x_{1:t-1}]\mathbb{P}_{BO}(r_{t-1}|x_{1:t-1})-\sum_{r_{t-1}=0}^{t-2}\mathbb{E}[x_t|x_{t-1}^{(r)}]\mathbb{P}_{BO}(r_{t-1}|x_{1:t-1})\right\vert\\ 
    &\quad\ge\left\vert\epsilon_L\sum_{r_{t-1}=0}^{t-2}\mathbb{P}_{BO}(r_{t-1}|x_{1:t-1})\right\vert=\epsilon_L\cdot\alpha_{t-1}.
\end{align*}
In the case $\mathfrak{T}_{GLRT}^*=1$, the loss of CBOCPD is written as 
\begin{align*}
    &|\mathbb{E}[x_t|x_{1:t-1}]-\mathbb{E}_{CBO}[x_t|x_{1:t-1}]|=|\mathbb{E}[x_t|x_{1:t-1}]-\mathbb{E}[x_t|\emptyset]|\le\epsilon_U. 
\end{align*}
Then, the gain is bounded as
\begin{align*}
    &|\mathbb{E}[x_t|x_{1:t-1}]-\mathbb{E}_{BO}[x_t|x_{1:t-1}]|-|\mathbb{E}[x_t|x_{1:t-1}]-\mathbb{E}_{CBO}[x_t|x_{1:t-1}]| \ge \epsilon_L\cdot\alpha_{t-1}-\epsilon_U.
\end{align*}
In the case $\mathfrak{T}_{GLRT}^*=0$, the loss of CBOCPD is written as 
\begin{align*}
    &|\mathbb{E}[x_t|x_{1:t-1}]-\sum_{r_{t-1}=1}^{t-1}\mathbb{E}[x_t|x_{t-1}^{(r)}]\cdot \mathbb{P}_{CBO}(r_{t-1}|x_{1:t-1})|\\
    &\quad=\left\vert\sum_{r_{t-1}=1}^{t-1}\mathbb{E}[x_t|x_{1:t-1}]\mathbb{P}_{CBO}(r_{t-1}|x_{1:t-1})-\sum_{r_{t-1}=1}^{t-1}\mathbb{E}[x_t|x_{t-1}^{(r)}]\mathbb{P}_{CBO}(r_{t-1}|x_{1:t-1})\right\vert\\\
    &\quad=\left\vert\sum_{r_{t-1}=1}^{t-2}\mathbb{E}[x_t|x_{1:t-1}]\mathbb{P}_{CBO}(r_{t-1}|x_{1:t-1})-\sum_{r_{t-1}=1}^{t-2}\mathbb{E}[x_t|x_{t-1}^{(r)}]\mathbb{P}_{CBO}(r_{t-1}|x_{1:t-1})\right\vert\\
    &\quad\le\left\vert\epsilon_L\sum_{r_{t-1}=1}^{t-2}\mathbb{P}_{CBO}(r_{t-1}|x_{1:t-1})\right\vert=\epsilon_U\cdot\beta_{t-1}.
\end{align*}
Then, the gain is bounded as
\begin{align*}
    &|\mathbb{E}[x_t|x_{1:t-1}]-\mathbb{E}_{BO}[x_t|x_{1:t-1}]|-|\mathbb{E}[x_t|x_{1:t-1}]-\mathbb{E}_{CBO}[x_t|x_{1:t-1}]|\ge \epsilon_L\cdot\alpha_{t-1}-\epsilon_U\cdot\beta_{t-1}.
\end{align*}
As $\mathbb{P}(\mathfrak{T}_{GLRT}^*=1)=\delta_0^{\text{\rom{1}}}\delta_1^{\text{\rom{1}}}$ and $\mathbb{P}(\mathfrak{T}_{GLRT}^*=0)=(1-\delta_0^{\text{\rom{1}}})(1-\delta_1^{\text{\rom{1}}})$ in stationary case, the expected gain is bounded from below as
\begin{align*}
&\mathbb{E}(|\mathbb{E}[x_t|x_{1:t-1}]-\mathbb{E}_{BO}[x_t|x_{1:t-1}]|-|\mathbb{E}[x_t|x_{1:t-1}]-\mathbb{E}_{CBO}[x_t|x_{1:t-1}]|)\\
&\quad\ge(\epsilon_L\alpha_{t-1}-\epsilon_U)\delta_0^{\text{\rom{1}}}\delta_1^{\text{\rom{1}}}+(\epsilon_L\alpha_{t-1}-\epsilon_U\beta_{t-1})(1-\delta_0^{\text{\rom{1}}})(1-\delta_1^{\text{\rom{1}}})\ge0
\end{align*}
where the last inequality follows the assumption. Thus we can conclude that the expected gain is non-negative.

\end{proof}

\end{document}